\documentclass[runningheads]{llncs}

\usepackage[T1]{fontenc}
\usepackage{amsmath,amssymb}
\usepackage{mathtools}
\usepackage{bbm}

\usepackage{scalerel} %

\usepackage{times}
\usepackage{helvet}
\usepackage{courier}
\usepackage{orcidlink}

\usepackage{graphicx}
\usepackage{subcaption}
\graphicspath{{fig/}{../fig/}{fig/exp_heatmap/}{fig/heat1preproc}{fig/heat22scores}}
\usepackage[ruled,vlined,linesnumbered]{algorithm2e} 

\usepackage[appendix=append]{apxproof}

\usepackage{cleveref}
\usepackage{autonum} 

\usepackage {booktabs}

\makeatletter 
\@ifclassloaded{llncs}{%
}{%
  \usepackage{orcidlink}

\usepackage{graphicx}
\usepackage{amsmath,amssymb}
\usepackage{mathtools}
\mathtoolsset{showonlyrefs=false} 
\usepackage[ruled,vlined,linesnumbered]{algorithm2e} 
\usepackage{cite}
\usepackage{bbold} 

\usepackage{caption} 
\captionsetup{belowskip=0pt}
\usepackage[skip=0.5ex]{subcaption} 
\usepackage{url}
\usepackage{bm}
\usepackage{bbm}
\usepackage{textcomp}
\usepackage{wrapfig}

\usepackage{tikz}
\usetikzlibrary{backgrounds, positioning}

\newsavebox{\cmbox}

\newtheorem{theorem}{Theorem}

\newtheorem{lemma}{Lemma}

\newtheorem{definition}{Definition} 
\crefname{definition}{Def.}{Defs.}

\newtheorem{problema}{Problem} 

\newcounter{claimnum}

\usepackage{cleveref}
\crefname{definition}{Definition}{Definitions}
\crefname{lemma}{Lemma}{Lemmas}
\crefname{proposition}{Proposition}{Propositions}
\crefname{theorem}{Theorem}{Theorems}
\crefname{remark}{Remark}{Remarks}
\crefname{problem}{Problem}{Problems}
\crefname{condition}{Condition}{Conditions}
\crefname{assumption}{Assumption}{Assumptions}
\crefname{table}{Table}{Tables}
\crefname{figure}{Fig.}{Figures}
\crefname{equation}{Eq.}{Equations}
\crefname{section}{Sec.}{Sections}
\crefname{chapter}{Chapter}{Chapters}
\crefname{algocf}{Algorithm}{Algorithms}
\Crefname{algocf}{Algorithm}{Algorithms}

  \newcommand{\authorrunning}[1]{}
  \newcommand{\institute}[1]{\date{#1}}
  \newcommand{\email}{}
  \newcommand{\inst}[1]{}
  \newcommand{\keywords}[1]{\textit{keywords}: #1}

}
\makeatother 

\usepackage[inline,shortlabels]{enumitem} 
\setlist{itemsep=0.0pc}
\setlist[enumerate,1]{ label= (\arabic*), ref=\arabic*}
\setlist[enumerate,2]{ label= (\roman*),ref  = \roman*}
\setlist[enumerate,3]{ label= (\alph*), ref  = (\alph*)}
\setlist[itemize,1]{label=--}
\usepackage{cite}

\def\idrm#1{\ensuremath{\mathrm{#1}}}
\def\idtt#1{\ensuremath{\mathtt{#1}}}

\normalfont 

\newcommand{\sete}[1]{ \{\kern.25em{#1}\kern.25em \}}
\newcommand{\set}[1]{\{\kern0.00em#1\kern0.00em\}} 
\newcommand{\lst}[1]{[\kern0.05em#1\kern0.05em]}

\newcommand{\by}{\times}
\newcommand{\pair}[1]{\langle #1\rangle}

\newcommand{\pow}[1]{2^{#1}}

\makeatletter
\DeclareRobustCommand\midop[1]{%
  \mathop{\vphantom{#1}\mathpalette\midop@{#1}}\slimits@
}
\newcommand{\midop@}[2]{%
  \vcenter{%
    \sbox\z@{$#1\sum$}%
    \hbox{\resizebox{\ifx#1\displaystyle0.475\fi\dimexpr\ht\z@+\dp\z@}{!}{$\m@th#2$}}%
  }%
}
\makeatother
\makeatletter
\DeclareRobustCommand\bigop[1]{%
  \mathop{\vphantom{#1}\mathpalette\bigop@{#1}}\slimits@
}
\newcommand{\bigop@}[2]{%
  \vcenter{%
    \sbox\z@{$#1\sum$}%
    \hbox{\resizebox{\ifx#1\displaystyle0.8\fi\dimexpr\ht\z@+\dp\z@}{!}{$\m@th#2$}}%
  }%
}
\makeatother
\makeatletter
\DeclareRobustCommand\mybigopin[1]{%
  \mathop{\vphantom{\sum}\mathpalette\mybigopin@{#1}}\slimits@
}
\newcommand{\mybigopin@}[2]{%
  \vcenter{%
    \sbox\z@{$#1\sum$}%
    \hbox{\resizebox{\ifx#1\displaystyle1.0\fi\dimexpr\ht\z@+\dp\z@}{!}{$\m@th#2$}}%
  }%
}
\DeclareRobustCommand\mymidopin[1]{%
  \mathop{\vphantom{\sum}\mathpalette\mymidopin@{#1}}\slimits@
}
\newcommand{\mymidopin@}[2]{%
  \vcenter{%
    \sbox\z@{$#1\sum$}%
    \hbox{\resizebox{\ifx#1\displaystyle0.25\fi\dimexpr\ht\z@+\dp\z@}{!}{$\m@th#2$}}%
  }%
}
\makeatother

\renewcommand{\paragraph}[1]{\textbf{#1}}

\renewcommand{\vec}[1]{\boldsymbol{#1}} 

\newcommand{\nat}{\mathbb{N}}

\newcommand{\rat}{\mathbb{R}}

\newcommand{\bat}{\mathbb{B}}

\newcommand{\Zero}{\mathbf{0}}
\newcommand{\One}{\mathbf{1}}

\newcommand{\ADTC}{\textsf{ADTC}}

\newcommand{\ud}[1]{\underline{#1}}
\newcommand{\vsigma}[1][{}]{\vec\sigma^{#1}}
\newcommand{\vm}{{\vec m}}
\newcommand{\A}{{\sig A}}
\newcommand{\TS}[1][{\M[\Delta]}]{{{#1}\times \pow{\sig X}}}
\renewcommand{\S}{{\sig S}}
\newcommand{\HOM}{\idtt{Hom}}

\newcommand{\idx}[3]{_{#1=#2}^{#3}} 

\newcommand{\medstrut}{\rule{0pt}{1pc}}
\newcommand{\iFor}[2]{\textbf{for} {#1} \textbf{do}\hspace{0.25em}{\relax #2}}

\newcommand{\sig}[1]{\mathcal{#1}}
\newcommand{\proc}[1]{\pb{#1}}
\newcommand{\op}[1]{\mathtt{#1}} 
\newcommand{\pb}[1]{{\mbox{\textsf{#1}}}} 
\newcommand{\stk}[1]{\substack{#1}}

\newcommand{\R}{\sig R} 

\newcommand{\M}[1][\Delta]{\sig M_{#1}} 

\newcommand{\prodseq}[4]{{#1}_{#3}\kern-.5pt\times\dots\times\kern-.5pt{#1}_{#4}}

\newcommand{\Prod}{\raisebox{-.99pt}{\scalerel*{\prod}{\overline{M}}}\kern0pt}

\newcommand{\bigmul}{\bigotimes}
\newcommand{\setmn}[1]{\setminus\set{#1}}

\newcommand{\Data}{\pow\D}

\newcommand{\X}{\sig X}
\newcommand{\D}{\X}

\renewcommand{\S}{\sig S}

\newcommand{\size}{\op{size}}

\newcommand{\infrac}[2]{({#1}/{#2})}

\newcommand{\nd}[1]{\langle #1\rangle} 
\renewcommand{\=}{\kern-2pt=\kern-2pt}

\newcommand{\Ttensor}[1][]{\mathbb{T}[\vec\M, \R]}

\newcommand{\Dint}[2]{[#1\hskip1pt\raise-0pt\hbox{$:$\hskip1pt}#2]}

\crefname{theorem}{Theorem}{Theorems}
\crefname{cor}{Corollary}{Corollaries}
\crefname{proposition}{Proposition}{Propositions}
\crefname{lemma}{Lemma}{Lemmas}
\crefname{property}{Property}{Properties}
\crefname{condition}{Condition}{Conditions}
\crefname{fact}{Fact}{Facts}
\crefname{problem}{Problem}{Problems}
\crefname{remark}{Remark}{Remarks}
\crefname{algorithm}{Algorithm}{Algorithms}
\crefname{section}{Sec.}{Secs.}
\crefname{subsection}{Sec.}{Secs.}
\crefname{proof}{Proof}{Proofs}

\begin{document}

\title{Algebraic Model Counting for Global Analysis of Optimal Decision Trees}
\titlerunning{Algebraic Model Counting for Optimal Decision Trees}
\author{Hiroki Arimura}
\institute{Hokkaido University, N14 W9, Kita-ku, Sapporo, Hokkaido, 060-0814 Japan
  \email{arim@ist.hokudai.ac.jp}
}


\maketitle

\begin{abstract}
Ensuring model reliability in Explainable AI requires a global assessment of the hypothesis space.
We propose a formal framework for the exhaustive analysis of optimal and near-optimal decision trees, called \textit{Algebraic Decision Tree Counting} ({ADTC}). 
Inspired by Algebraic Model Counting (AMC) in knowledge representation, ADTC reformulates diverse analytical tasks, such as optimization, counting, and sampling, into a unified sum-of-products computation over a semiring $R$.
While the
hypothesis space of decision trees is doubly exponential with respect to the maximum depth $\Delta$, our dynamic programming algorithm achieves $O^*(n^{O(\Delta)})$ time complexity in the number of features $n$, where $O^*$ suppresses polynomial factors.
To handle complex constraints consisting of multiple tree metrics, we introduce \textit{model behavior tensors} that aggregate semiring values via convolution products over a tensor semiring. This algebraic approach efficiently constructs a model profile that
captures 
the global landscape and trade-offs between criteria such as accuracy, size, and fairness. We demonstrate the utility of our software, 
\texttt{emtrees}, 
on real-world datasets, illustrating how ADTC facilitates evidence-based model selection in sensitive domains.


  
\end{abstract}

\section{Introduction}
\label{sec:intro}

\newcommand{\profilewidth}{0.42\linewidth}


\textit{Explainable AI} (XAI) \cite{rudin2022interpretable} has placed a growing emphasis on \textit{decision trees} \cite{breiman:etal1984classification} due to their inherent interpretability. However, the phenomenon of \textit{predictive multiplicity}~\cite{black2022model,Marx:ICML2020}, where multiple models achieve similar performance but offer different explanations, poses a significant challenge for model reliability.
To address this, we propose \textit{Algebraic Decision Tree Counting} (\pb{ADTC}), a framework that performs a global assessment of the entire hypothesis space rather than identifying a single heuristic solution.

Our \pb{ADTC} is inspired by \textit{algebraic model counting} (\pb{AMC})
proposed by De Raedt and Kimmig \cite{kimmig2017algebraic} in knowledge representation \cite{darwiche2002knowledge}.
AMC is an instance of the algebraic computation paradigm (see, e.g., Aji and McEliece \cite{aji:mceliece:ieeeit2002generalized}, Eiter and Kiesel \cite{eiter2023semiring}, and Goral \cite{goral:giesen:blacher:staudt:kaus:aaai2024model}) that generalizes logical counting by evaluating formulas over a semiring.
Specifically, our framework aggregates behaviors of {near-optimal prediction models} based on a dataset and a constraint formula, whereas AMC aggregates weights of satisfying assignments for a given formula.
In this context, \textit{near-optimal models} (or \textit{good models}) \cite{nijssen:fromont:dl8:dmkd:2010,demirovic:aaai2021nonlinear:metrics,lin2020gosdt:generalized,Marx:ICML2020} refer to those within the \textit{Rashomon set} \cite{breiman:statsci2001twocultures,rudin2022interpretable,xin2022treefarm:exploring} that satisfy a predefined performance threshold relative to the empirical best model.

\begin{figure}[t]
  \centering
  \begin{subfigure}[t]{\profilewidth}
    \centering
    \includegraphics[width=0.95\linewidth]{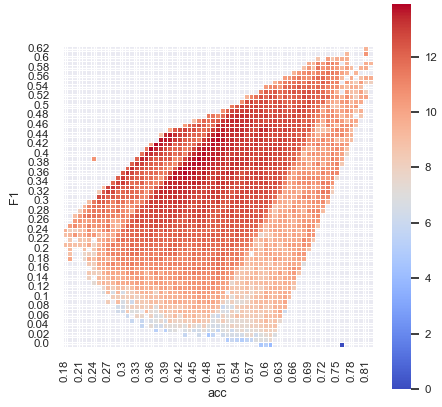}
    \caption{Accuracy vs. F1-score.}
    \label{fig:profile:acc_f1}
  \end{subfigure}
  \hfill
  \begin{subfigure}[t]{\profilewidth}
    \centering
    \includegraphics[width=0.95\linewidth]{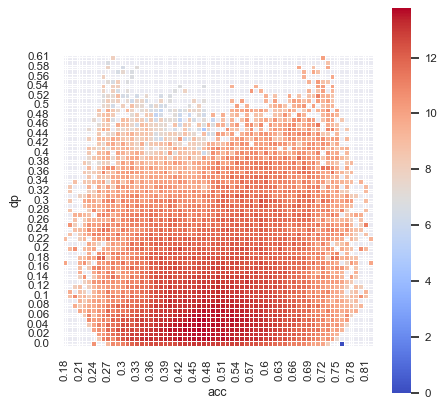}
    \caption{Accuracy vs. DP gap.}
    \label{fig:profile:acc_dp}
  \end{subfigure}
  \vspace{-1mm}
  \caption{Model profiles of 34,706 decision trees
    with $\texttt{maxdep}=5$ and $\texttt{relminsup}=0.15$ 
    from the Rashomon set on the \texttt{adult} dataset
    generated by \pb{ADTC}, showing
    (a) trade-off between Accuracy and F1-score, and (b) model multiplicity between Accuracy and DP-gap.
    See Fig.~\ref{fig:profile:acc-eodd:acc-eopp} for more examples
    and \cref{subsec:exp:tradeoff} for explanation.
  }
  \label{fig:profile:acc-fone:acc-dpgap}
\end{figure}

Conceptually, the framework \pb{ADTC} executes an ``\textit{enumerate-\allowbreak filter-\allowbreak aggregate}'' style query over the huge hypothesis space of decision trees, similar to
analytic queries in relational databases \cite{bakibayev2012fdb}.
While the number of syntactically distinct decision trees is $n^{\Theta(2^\Delta)}$, our approach achieves a time complexity of $O^*(n^{\Delta})$, which is polynomial in data size $n$ but exponential in query size $\Delta$, by employing a dynamic programming scheme based on tensor operations. Consequently, our framework \textit{avoids the double-exponential complexity} of exhaustive enumeration through algebraic aggregation.

The primary contributions of this paper are summarized as follows:
\begin{itemize}
\item \paragraph{Algebraic framework for decision trees}:
  In \cref{sec:preliminaries}, we propose \pb{ADTC}, a unified framework based on algebraic model counting~\cite{kimmig2017algebraic,eiter2023semiring}, to evaluate and aggregate decision trees over arbitrary commutative semirings.
  It provides a comprehensive perspective that generalizes previous studies on
  \textit{induction} \cite{mehta:raghavan2002decision,nijssen:fromont:dl8:dmkd:2010,aglin2020learning:dl8},
  \textit{enumeration}\cite{ruggieri2017enumerating},
  \textit{counting} as well as \textit{sampling}~\cite{arimura:osabe:uno2017ifors,xin2022treefarm:exploring}, and
  \textit{Pareto optimization}~\cite{demirovic:aaai2021nonlinear:metrics}
  of optimal decision trees, which have often been treated independently.

\item \paragraph{Theoretical complexity guarantees}:
  In \cref{sec:method}, we develop a dynamic programming (DP) algorithm {\proc{EMT}} using tensor operations (Algorithm~\ref{algo:emt}) that achieves a running time of $O^*(n^{O(\Delta)})$ (\cref{thm:algo:emt:main}). This complexity is polynomial in data size $n$, although exponential in query size $\Delta$,
  avoiding the double-exponential bottleneck of exhaustive enumeration.
  The key to this efficiency lies in aggregating information of multiple metrics,
  such as \texttt{size} and \texttt{error} of decision trees, 
  into
  a semiring of \textit{model behavior tensors}, that allow for efficient computation via DP over a \textit{factorized representation} of the decision tree space.
  In \cref{sec:ext}, we extend
  \proc{EMT}
  to \textit{nonlinear metrics}
  \cite{demirovic:aaai2021nonlinear:metrics}
  as well as
  \textit{model selection and sampling}~\cite{arimura:osabe:uno2017ifors,goral:giesen:blacher:staudt:kaus:aaai2024model}. 

\item \paragraph{Transparent model profiling}:
  In \cref{sec:exp}, we demonstrate that our framework enables global navigation of the rashomon set, facilitating transparent and rigorous analysis of
  \textit{trade-offs between multiple objectives}, such as the \textit{accuracy} and \textit{F1-score} in Fig.~\ref{fig:profile:acc-fone:acc-dpgap}, 
  and  
  \textit{assessment of preprocessing},
  such as the \textit{original} and \textit{balanced accuracy} in Fig.~\ref{fig:profile:acc:bacc}
  (in \cref{sec:exp}). 
  This profiling addresses diverse requirements in real-world machine learning deployments.
  Finally, we conduct empirical evaluations of our \texttt{emtrees} software~\cite{emtrees26}, an implementation of the \proc{EMT} algorithm 
  in Secs.~\ref{sec:method} and \ref{sec:ext},
  on the standard benchmark UCI \texttt{adult} dataset to demonstrate its scalability and effectiveness in global analysis tasks for sensitive domains.
\end{itemize}

Overall, this proposed framework enables a global navigation of the Rashomon set, allowing \pb{ADTC} to analyze the trade-offs between multiple objectives, such as accuracy and fairness. By providing a rigorous profile of all near-optimal models, our framework facilitates transparent and rigorous model selection for diverse requirements in real-world machine learning deployments.

\begin{figure}[t]
  \centering
  \begin{subfigure}[t]{\profilewidth}
    \centering
    \includegraphics[width=\linewidth]{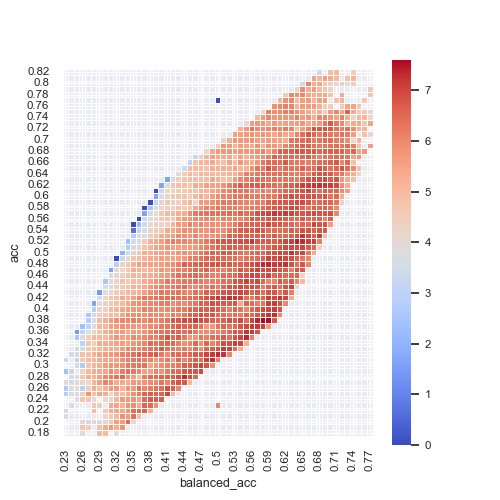}
    \caption{Original \texttt{adult} dataset.}
    \label{fig:profile:acc:bacc:orig}
  \end{subfigure}
  \hfill
  \begin{subfigure}[t]{\profilewidth}
    \centering
    \includegraphics[width=\linewidth]{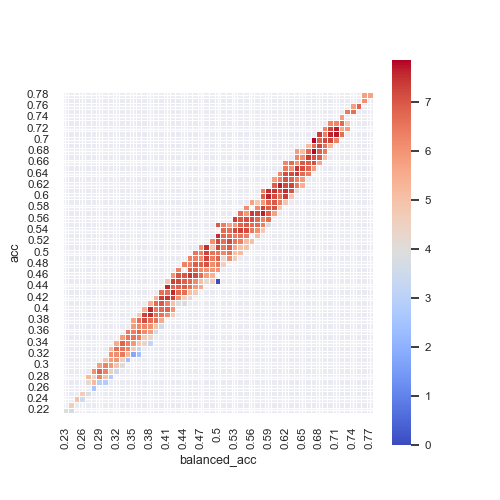}
    \caption{Balanced version of \texttt{adult} dataset.}
    \label{fig:profile:acc:bacc:bal}
  \end{subfigure}

  \vspace{-1mm} 
  \caption{Model profiles of 34,706 decision trees
    with $\texttt{maxdep}=5$ and $\texttt{relminsup}=0.15$ 
    from the Rashomon set on the \texttt{adult} dataset
    generated by \pb{ADTC}
    on assessment of preprocessing with accuracy and balanced accuracy. 
    See \cref{subsec:exp:assessment} for explanation. 
  }
  \label{fig:profile:acc:bacc}
\end{figure}

\subsection{Related work}
The global analysis of decision tree spaces intersects with several research areas, including optimal decision tree construction, model counting, and the study of the Rashomon set.

\paragraph{Optimal decision tree construction}. Recent advancements have introduced exact algorithms such as \pb{GOSDT} \cite{lin2020gosdt:generalized} and \pb{DL8.5} \cite{aglin2020learning:dl8}. While these methods excel at finding a single optimal tree
\cite{nijssen:fromont:dl8:dmkd:2010,mehta:raghavan2002decision},
they are not designed to characterize the entire landscape of near-optimal models. In contrast, our approach enables a comprehensive aggregation of these models through a unified algebraic framework.

\paragraph{The Rashomon set and predictive multiplicity}. 
Previous work by Dong and Rudin \cite{Dong:NMI2020} explored the Rashomon set \cite{rudin2022interpretable} via \textit{sampling}, while Xin \textit{et al.}~\cite{xin2022treefarm:exploring} utilized \textit{compressed representations}. However, the latter requires \textit{exponential memory} and is primarily limited to counting~\cite{xin2022treefarm:exploring}. In contrast, our approach employs model behavior tensors computable in \textit{polynomial space}
(see~\cref{thm:algo:emt:main}).
Within a unified semiring-based framework of \pb{ADTC}, we provide an \textit{exact profile} of model diversity that reveals the underlying \textit{predictive multiplicity} and \textit{trade-offs}, enabling various analytical tasks beyond simple counting. 

\paragraph{Algebraic model counting}.
Our work bridges the gap between \textit{Algebraic Model Counting} (\pb{AMC}) \cite{kimmig2017algebraic} and the \textit{combinatorial space of decision trees} by adapting the \pb{AMC} paradigm to the \textit{recursive structure of trees}. This allows for efficient \textit{sum-of-product computations}~\cite{eiter2023semiring,goral:giesen:blacher:staudt:kaus:aaai2024model} over tensors, avoiding the \textit{double-exponential bottlenecks} associated with naive enumeration.


\section{Preliminaries}
\label{sec:preliminaries}

The following notation and terminology are used throughout this paper. 
For standard terminology not defined below, please consult \cite{hastie2001eslbook} for machine learning, \cite{Cormen:Stein:Rivest:Leiserson:2009} for algorithms, \cite{ebbinghaus:flum:thomas:1994:mathematical:logic} for logic, and \cite{burgisser2013algebraic} for algebraic computation.

\subsection{Basic definitions}
\paragraph{Numbers, sets, and vectors.} 
Let $\mathbb{R}$, $\mathbb{Z}$, and $\mathbb{N} = \{0, 1, \dots\}$ be the sets of all real numbers, integers, and natural numbers, respectively. 
For an integer $n, i, j\in \mathbb{N}$ with $i\le j$, we define
$[n]=\{1, \dots, n\}$, and $[i..j] = \set{i, i+1, \dots, j}\subseteq \mathbb{N}$.
For any sets $A$ and $B$, $|A|$ denotes the \textit{cardinality} of $A$, $2^A$ the power set of $A$, and $A^*$ the set of all finite sequences of elements from $A$. 
The notation $B^A$ denotes the set of all mappings from $A$ to $B$.
For $d \in \nat$, we write $d$-vector $\vec x = (x_j)\idx j1d$ as $(x_j)_d$.
For $d \in \mathbb{N}$ and a $d$-vector $\vm = (m_1, \dots, m_d) \in \mathbb{N}^d$, let $\mathbb{N}[\vm]$
denote
$[m_1] \times \cdots \times [m_d] \subseteq \mathbb{N}^d$.

\paragraph{Semirings and polynomials.}
A \textit{monoid} $(M, \cdot, e)$ is a set M equipped with an associative binary operation $\cdot$ and an identity $e \in M$. It is commutative if $a + b = b + a$ holds. 
A \textit{semiring} is an algebraic structure $\mathcal{R} = (R, +, \cdot, \Zero, \One)$, where $a\cdot b$ is written $ab$, and 
\begin{enumerate*}[label=(\roman*)]
\item $(R, +, \Zero)$ is a commutative monoid,
\item $(R, \cdot, \One)$ is a monoid,
\item The multiplication $\cdot$ \textit{distributes} over addition $+$: $a (b + c) = (a b) + (a c)$ and $(a + b) c = (a c) + (b c)$, and 
\item $\Zero$ is \textit{absorbing}, i.e., $\Zero a = a \Zero = \Zero$
\end{enumerate*}
for all $a,b,c \in R$.
In this work, we mainly consider commutative semirings.
For $d \in \mathbb{N}$ and a semiring $R$, we denote by $R[(x_j)_d]$ the semiring of $d$-variate polynomials with coefficients in $R$.
Each polynomial 
in $R[(x_j)_d]$
with indeterminates $x = (x_j)_d$ is represented as
\begin{math}
  p(x)
  = \sum_{\vec{v}} c_{\vec{v}} x^{\vec k}
  = \sum_{\vec{v} \in \mathbb{N}^d} c_{\vec{v}} \prod_{j=1}^d x_j^{v_j}, 
\end{math}
where
$c_{\vec{v}} \in R$. 
These structures are fundamental for \textit{algebraic model counting} \cite{kimmig2017algebraic}.

\subsection{Prediction models and their behaviors}

\paragraph{Decision trees.}
Let $\sig{X}$ be a universe of \textit{data}.
We assume $n$ \textit{Boolean features} $F = \{f_1$, $\dots$, $f_n\}$ and $c$ categorical \textit{labels} $L = \set{y_0, \dots, y_{c-1}}$, where each {Boolean feature} is a mapping $f: \sig{X} \to \set{0,1}$.
A \textit{decision tree} (a \textit{tree}, for short) over alphabets $\Sigma = (F, L)$ is an expression represented as a node-labeled binary tree.
A tree $t$ and its \textit{string notation} are defined inductively as follows: $t$ is either
\begin{enumerate*}[label=(\roman*)]
  \item a leaf labeled with $y \in L$, denoted by $y$ itself, or
  \item a composite tree with a root labeled with a feature $f$ in $F$ having two children $t_0$ and $t_1$, denoted by $\langle f, t_0, t_1 \rangle$.
  \end{enumerate*}
  A tree $t$ induces a \textit{prediction function} $f_t: \mathcal{X} \to \mathcal{Y}$ that, given a data $x \in \sig X$, returns a prediction label $y = f_t(x) \in L$ via a standard root-to-leaf traversal. A \textit{sample} is a set of data $S = \set{x_i}\idx i1m$ in $\Data$. An \textit{input data} is a pair $(S, y)$ of a sample $S$ and a \textit{labeling function} $y: S \to L$.

  \paragraph{Hypothesis spaces}: We consider the hypothesis space $\mathcal{M}_{\Delta}$
  of \textit{all decision trees with depth at most $\Delta \ge 0$} over a given alphabet $\Sigma = (F, L)$
  because complexity is primarily governed by depth~\cite{mehta:raghavan2002decision}.
  The number of syntactically distinct Boolean trees in $\mathcal{M}_{\Delta}$ is $n^{\Theta(2^{\Delta})}$, which exhibits double-exponential growth.
  While $\M[\Delta]$ serves as our baseline, we can further restrict it by imposing
  $\texttt{minsup} \ge 1$ (\textit{minimum support})
  and $\texttt{nbins}$ (discretization bins), denoted as
  $\mathcal{M}_{\Delta, \texttt{minsup}, \texttt{nbins}}$
  for effective pruning.

\subsection{Model behavior metrics}
We will analyze the hypothesis space $\M$ of models by means of measuring functions, called \textit{model metrics} (or \textit{metrics}). 
Let $\Delta \in \nat$.
In our global analysis of decision tree space $\M$, a smallest unit of analysis is a pair $(t, S)$ of a tree $t$ and a sample $S$, called a \textit{structure}. The domain of structures is $\S = \TS$. 
Then, a \textit{model metrics} (or \textit{metrics}) is any mapping $\sigma: \S \to \idtt{dom}(\sigma)$ that assigns a value $k = \sigma(t, S) \in \S$ to each model $t \in \M$, where $\idtt{dom}(\sigma)$ is any set.
A metric $\sigma$ is said to be \textit{primitive} if it has an integer-valued range $\idtt{dom}(\sigma) = [0..m_\sigma]$, where $m_\sigma \in \nat$ is called the \textit{maximum range}.
The following are examples of primitive metrics used in this paper. 
\begin{enumerate*}[label=(\roman*)]
\item \textit{Structural metrics}: The \textit{size} $size(t) = |Lv(t)|$ and \textit{depth} $dep(t)$, respectively, with ranges $1 \le size(t) \le 2^\Delta$ and $0 \le dep(t) \le \Delta$. 

\item \textit{Semantic metrics}: Given an input data $D = (S, y)$, the \textit{error} $err(t, S) = \sum_{x \in S} \One[\phi_t(x) \neq y(x)]$ measures predictive performance, and the \textit{support} $supp(t, S) = \min_{\lambda \in Lv(t)} |S(\lambda)|$ captures the statistical reliability of leaf nodes.
\end{enumerate*}
In \cref{sec:ext:nonlinear:metrics}, we will introduce primitive metrics related to contingency tables as well as complex nonlinear metrics.

To devise efficient algorithms for \pb{ADTC}, we need to introduce the notion of decomposition of a metric, in a similar way to \pb{AMC}~\cite{kimmig2017algebraic,goral:giesen:blacher:staudt:kaus:aaai2024model}.
For any feature $f \in F$, we define the \textit{split} of a data set $S$ by $f$ to be the partition $S = S_{f=1} \uplus S_{f=0}$ such that  $S_{f=i} = \{ x \in S \mid f(x) = i \}$ for each $i \in \{0, 1\}$.


\begin{definition}[Decomposable metrics]\rm
  A function $f: \TS \to \nat$ over the domain of structures $\TS$ is \textit{decomposable} on tree structures if there exists some monoid $(\nat, \circ, 0)$ such that $f$ satisfies the recurrence:
\begin{align}\label{eq:1}
  f(t, S) =
  \begin{cases}
    \ud{f}(y, S)
    & \text{if } t = y \in L, \\
    f(t_1, S_{f=1}) \circ f(t_0, S_{f=0})
    & \text{if } t = \pair{f, t_1, t_0}, S = S_{f=1} \uplus S_{f=0}, 
  \end{cases}
\end{align}
where
\begin{enumerate*}[label=(\roman*)]
\item $\ud{f}: \TS[L] \to \nat$ is the restriction of $f$ to $\TS[L]$, called the \textit{(leaf) labeling function}, and
\item $\circ:\nat^2\to \nat$ is a binary operator over $\nat$. 
\end{enumerate*}
Then, we say that $f$ is \textit{decomposable via operator} $\circ$, denoted $f = \HOM[\ud{f}, \circ]$. 
\end{definition}

\subsection{Constraint formulas and $\phi$-good models}
\label{subsec:prelim:constraint}
In the global analysis of decision trees, it is standard practice to investigate a subspace of the hypothesis space consisting of high-quality models, or \textit{good models}~\cite{rudin2022interpretable,Marx:ICML2020}, called the \textit{Rashomon set}, such that $size(t) \le s$ and $err(t, S) \le e$. We generalize this notion below. 
Let $\Delta, d \in \nat$ be any integers. We assume a $d$-vector $\vm \in \nat^d$ and a structure domain $\S = \TS$. We assume a $d$-vector of \textit{primitive metrics} $\vsigma = (\sigma_i)_d$ and the associated \textit{shape vector} $\vm = (m_i)_d \in \nat^d$, where $\sigma_i: \S \to [0..m_i]$ for each $i \in [d]$.

We introduce the \textit{syntax} of constraint formulas as follows. 
The \textit{vocabulary} with metrics $\vsigma$ includes constants for numbers in $\nat$ and $\rat$, nullary function variables $(\sigma_i)_d$, binary operators
$+, \times, (\frac{\cdot}{\cdot})$,
and a binary relation symbol $\le$.
Terms are rational expressions%
\footnote{
  A \textit{rational expression} is a fraction of multi-variate polynomials
  $\frac{P(x)}{Q(x)}$, where $x$ is a sequences of variables for functions.
  The set of rational functions is closed under addition, subtraction, multiplication, and division by non-zero functions~\cite{burgisser2013algebraic}. 
}
constructed from all constants, operators, and function variables $\ud{\vsigma}$.
An atomic formula is either  $(f(\ud{\vsigma})\le c)$ or $(f(\ud{\vsigma})\le g(\ud{\vsigma}))$ with rational expressions $f,g$ and constants $c$.
Then, the set of \textit{constraint formulas} consists of all formulas with free variables $\ud{\vsigma}$ constructed from atomic formulas $\phi[\ud{\vsigma}]$ using Boolean operations $\neg, \lor, \land$.
For example, the followings are constraint formulas with metrics,  
\begin{align}
  \phi_1 &=  (\size \le 4) \land (\op{error} \le 0.05\cdot N)
  \text{ with } acc \equiv (N_{1,1} + N_{0,0})/N, 
  \\
  \phi_2 &=  (\op{acc} \ge 0.9) \land (\op{f1}\ge 0.76)
  \text{ with } \op{f1} \equiv (2\cdot N_{1,1})/(2\cdot N_{1,1} + N_{0,0} + N_{1,0}), 
\end{align}
where $\phi_1$ states that $t$ is accurate and succinct on $S$, while $\phi_2$ states that $t$ has high scores in both accuracy and F1.
Metrics $acc, f1, \set{N_{i,j}}$ will be introduced in \cref{sec:ext:nonlinear:metrics}.

The \textit{semantics} is defined as follows.
We consider a pair $(t, S) \in S = \TS$ as a \textit{structure} over the vocabulary. 
Then, a $d$-metric vector $\vsigma = (\sigma_i)_d$ assigns to the structure $(t, S)$ a $d$-interger vector $\vec k = (k_i)_d = \vsigma(t, S) \in \nat^d$, called the \textit{index} of $t$, 
where $k_i = \sigma_i(t, S)$ for all $i$. 
Given a formula $\phi = \phi[\ud{\vsigma}]$, we write $(t, S)\models \phi$, if $\phi$ evaluates true under valuation $\ud{\vsigma} \mapsto \vsigma(t, S) \in \nat^d$ with the standard interpretation to logical connectives.
Then, we say that a decision tree $t \in \M$ is a \textit{$\phi$-good model} on $S$ if $(t, S)\models \phi$.
Note that the truth of $\phi$ on the structure $(t, S)$ is solely determined by the vector $\vec k = \vsigma(t, S)$ in $\nat^d$.
We define the \textit{hypothesis space of $\phi$-good models} by
\begin{math}
  \M(\phi, S) = \set{ t \in \M \mid  (t, S) \models \phi }. 
\end{math}



\subsection{Our problem: algebraic decision tree counting}
\label{subsec:problem}

From now on, we define our problem, $\pb{ADTC}$, the algebraic decision tree counting over a semiring $(R, +_R, \cdot_R, \Zero, \One)$. Let $\S = \TS$ be the domain of structures. 

\begin{definition}[Aggregation]\rm 
  An \textit{aggregation function} over $R$ is any function $\alpha: \S \to R$ that assigns an element $\alpha(t, S) \in R$ to each structure $(t, S) \in \S = \TS$.
\end{definition}

Then, a \textit{rank-$d$ query} is a tuple $Q = (\Delta, \alpha, \vsigma, \phi[\vsigma])$ consisting of
\begin{enumerate*}[label=(\roman*)]
\item an integer $\Delta \in \nat$,
\item an aggregation operator $\alpha$, 
\item a $d$-vector of primitive metrics $\vsigma$ with shape vector $\vm$, and
\item a constraint formula $\phi = \phi[\vsigma]$ over $\vsigma$. 
\end{enumerate*}
An \textit{input data} $D = (S, y)$ is a pair of a sample $S$ and a label function $y: S \to L$. We state our problem. 

\begin{definition}[$\pb{ADTC}$]\rm\label{def:problem:adtc} 
  The \textsc{Algebraic Decision Tree Counting} over a semiring $R$ is the problem of, 
  given a rank-$d$ query $Q = (\Delta, \alpha, \vsigma, \phi[\vsigma])$ and
  an input data $D = (S, y)$,
  computing the semiring element
  \begin{align} \label{eqn:adtc:main}
    \ADTC(Q, D)
    &=
      \sum_{\stk{t \in \sig M_\Delta:\: (t,S)\models \phi[\vsigma] }}
      \alpha(t, S)
    & \in R,  
  \end{align}
  that is, the \textit{summation of the aggregation value} $\alpha(t, S)$ with $+_R$ \textit{over all $\phi$-good trees} $t$ within $\M$ relative to a sample $S$. 
\end{definition}

By varying a semiring $(R, +_R, \cdot_R)$ and a query $Q$ as its components, the \pb{ADTC} problem can naturally formulate a wide range of global analytics tasks as follows.
Let $\op{one}(t, S) = \One$ be the aggregation function that always returns the constant $\One \in R$.

\begin{lemma}
  For any $\Sigma = (F, L)$, the framework \pb{ADTC} can solve the following tasks
  for the space $\M$ of decision trees
  by varying a semiring $R$ and a query $Q$ as follows: 
\begin{enumerate}[label=(\arabic*)]
\item The Boolean ring $(\bat, \lor, \land, 0, 1)$ with $\alpha = \op{one}$
  and $\phi_\idrm{good} =  (\mathtt{size}\!\le\! s) \land (\mathtt{error} \!\le\! e)$
  serves for deciding the existence of a small and accurate decision tree~{\rm \cite{nijssen:fromont:dl8:dmkd:2010}}.
  
\item The natural number ring $(\nat, +, \times, 0, 1)$ with $\alpha = \op{one}$ serves for counting all small and accurate decision trees on the arithmetic semiring on natural numbers~{\rm \cite{xin2022treefarm:exploring,arimura:osabe:uno2017ifors}}. 
  
\item The min-plus semiring $(\nat, \min, +, \infty, 0)$ serves for finding accurate tree minimizing the error $\alpha_\idrm{err}(t, S) := \op{error}(t, S)$~{\rm \cite{ruggieri2017enumerating,arimura:osabe:uno2017ifors}}.
  Remark that $\alpha_\idrm{err}$ is decomposable  as $\alpha_\idrm{err}(\pair{f, t_1, t_0}, S) = \alpha_\idrm{err}(t_1, S_{f=1}) + \alpha_\idrm{err}(S, t_0, S_{f=0})$ using $\cdot_R = +$.
\end{enumerate}
\end{lemma}

\begin{proofsketch}
  The proof is straightforward by discussions similar to  
  \cite{goral:giesen:blacher:staudt:kaus:aaai2024model,eiter2023semiring}.
  \qed 
\end{proofsketch}
  
We remark that the complexity of \pb{ADTC} depends crucially on an underlying semiring $R$. To be precise, we introduce the parameter $t_R$ and $s_R$ to be the \textit{worst-case time} and \textit{space complexities} for operations on $\R$. We observe that $\pb{ADTC}$ can be solved by a straightforward method according to Eq.\eqref{eqn:adtc:main} as follows: it first initializes a variable $r = \Zero \in R$, then, scans all trees $t$ in $\M$, where at each iteration, it evaluates $\vsigma(t, S)$, and adds the weight $\alpha(t, S) \in R$ to $r$  if $(t, S) \models \phi[\vsigma]$ holds. However, since $|\M| = n^{\Theta(2^\Delta)}$, this method requires
doubly exponential time in $\Delta$.


\section{Efficient Algorithm}
\label{sec:method}\label{sec:algo}

This section presents a dynamic programming approach solving \pb{ADTC} in $O^*(n^{O(\Delta)} \cdot t_R)$ time and $O^*(s_R)$ space over $\sig M_\Delta$.
We first formalize decomposability of metrics (Sec.~3.1), 
and then develop the \proc{EMT} algorithm for unconstrained \pb{ADTC} (Sec.~3.2).
Reducing the tensor construction \pb{MT} to this unconstrained setting (Sec.~3.3) yields our final algorithm for the general constrained \pb{ADTC} problem (Sec.~3.4).

\subsection{Assumptions on analytic queries} 
\label{subsec:mathod:notation}

Let $(R, +_R, \cdot_R, \Zero, \One)$ be a semiring and $(\A, \vec\circ, \One)$ be a $d$-vector $\set{(A_i, \circ_i, 1_i)}_d$ of monoids.
Throughout, we assume any rank-$d$ query $Q = (\Delta, \alpha, \vsigma, \phi[\vsigma])$ satisfies two conditions:
\begin{enumerate}[label=(\roman*)]
  \item The aggregation $\alpha: \S \to R$ is decomposable via $\cdot_R$, that is, $\alpha = \HOM[\ud{\alpha}, \cdot_R]$.
  \item The metrics $\vsigma: \S \to \nat[\vm]$ with shape $\vm \in \nat^d$ is decomposable via $\vec\circ$, that is, $\vsigma = \HOM[\ud{\vsigma}, \vec\circ]$.
\end{enumerate}
When the underlying algebraic structures are clear, we simply write $\alpha$ and $\vsigma$ to denote the \textit{decomposable schemas} $(\alpha, R)$ and $(\vsigma, \A)$, respectively.


\subsection{\pb{ADTC} without constraints}
\label{sec:algo:wo:constraint}

First, we present our basic algorithm \proc{EMT} for unconstrained queries of the form $Q = (\Delta, \alpha)$, where metrics $\vsigma$ and constraint formula $\phi[\vsigma]$ are empty. 
Although such restricted queries seem useless in practice, they will turn out to be useful as an important building block of a general $\pb{ADTC}$ algorithm of \cref{sec:method:mtensor} and \cref{sec:method:together}.

\begin{definition}[Algorithm {\rm \proc{EMT}}]\rm
  \label{def:algo:emt}
  We assume an unconstrained rank-$d$ query
  $Q = (\Delta, \alpha)$
  and
  an input data $D = (S, y)$. 
  Then, the procedure $\proc{EMT}$ is defined by the recurrence below by induction on $\Delta\ge 0$, 
  where arguments $S$ and $F$ are any subsets of initial a sample $S_0$ and a feature set $F_0$:
  \begin{enumerate}
  \item In the case with $\Delta = 0$, we let 
    \begin{math}
      \proc{EMT}(0, S, F) = 
     \left(
        \sum_{\ell \in L}
        \ud{\alpha}(\ell, S)
     \right). 
    \end{math}
  \item In the case with $\Delta \ge 1$, we let 
   \begin{flalign}
     \proc{EMT}(\Delta, S, F)
     & \; =\;
     \proc{EMT}(0, S, F)
     \;+_R\;       
     \left(
     \sum_{f \in F}
     \left(\begin{array}{l}
       \proc{EMT}(\Delta-1, S_{f=1}, F\setmn{f}) \\ \cdot_R\;
       \proc{EMT}(\Delta-1, S_{f=0}, F\setmn{f})
     \end{array}\right)
       \right).
   \end{flalign}   
  \end{enumerate}
\end{definition}

\begin{algorithm}[t]
\caption{
  The recursive procedure that, given an unconstrained query $Q = (\Delta, \alpha)$ and an input $D = (S, y)$, solves the unconstrained \textsc{Algebraic Decision Tree Counting} problem  for the class $\M$ of decision trees over alphabets $\Sigma = (F, L)$ and a semiring $(R, +_R, \cdot_R, \Zero, \One)$.
  It recursively searches the subspace of $\M$ specified by a triple $(\Delta, S, F)$ consisting of a maximum depth $\Delta$, a sample $S$, and a feature set $F$. 
}\label{algo:emt}
\textbf{Procedure} $\proc{EMT}(\Delta, S, F)$\tcp*{Process path $P$}
\Begin{
    $\alpha = \Zero$\;
    \For{each label $\ell \in L$ }{
      $\alpha_\idrm{leaf} \gets \One$\; 
      \iFor{ each $x \in S$ }{
        $\alpha_\idrm{leaf} \gets (\alpha_\idrm{leaf} \cdot_R \ud{\alpha}(\ell, x))$\; 
      }
      $\alpha \gets \alpha +_R \alpha_\idrm{leaf}$\;
    }
    \If{ $\Delta = 0$ }{
      \Return $\alpha$\; 
    }
    \For{each feature $f \in F$ }{
      $\alpha_1 \gets \proc{EMT}(\Delta-1, S_{f=1}, F\setminus\set{f})$\tcp*{Process path $P\cdot\nd{f=1}$}
      $\alpha_0 \gets \proc{EMT}(\Delta-1, S_{f=0}, F\setminus\set{f})$\tcp*{Process path $P\cdot\nd{f=0}$}
      $\alpha \gets \alpha +_R (\alpha_1 \cdot_R \alpha_0)$\; 
    }
    \Return $\alpha$\; 
}
\end{algorithm}


In \cref{algo:emt}, we present the pseudocode of \proc{EMT} that implements the recurrence of Def.~\ref{def:algo:emt}, where we assume that
$\alpha$ is \textit{homomorphic} on data, i.e., 
\begin{math}
  \ud{\alpha}(\ell, S) =
  \bigmul_{x \in S} \ud{\alpha}(\ell, \set{x})
\end{math}
as it is true with most $\alpha$ in this paper. 
In the top-level, given $\Sigma = (F, L)$ and an input $D = (S, y)$,
the invocation of 
$\pb{EMT}(\Delta, S, F)$ computes a semiring element $v \in R$. 
By the distributivity of $\cdot_R$  over $+_R$, we can show the next lemma.

\begin{lemma}[Correctness of \proc{EMT}]\label{lem:emt:correctness}
Given an unconstrained rank-$d$ query $Q = (\Delta, \alpha)$ and an input data $D = (S, y)$, the procedure $\proc{EMT}$ solves the unconstrained $\pb{ADTC}$ problem. 
\end{lemma}

\begin{proof}
  In what follows, we write the solution $a(\M, S, F)$ of $\pb{ADTC}$
  by emphasizing its dependency on $S, F$ in the recurrence. 
  Then, we show the claim that   $a(\M[\Delta], S, F)$ coincides with
  the return value $\proc{EMT}(\Delta, S, F)$ of the algorithm (*1). 

  (1) First, we suppose that $\Delta = 0$. Since $\M[0] = L$ and $\phi = 1$, we have the equations 
  \begin{align}
    & a(\M[0], S, F)
      \;=\; \sum_{\ell \in L}
      \ud{\alpha}(\ell, S)
    \;=\; \proc{EMT}(0, S, F), 
  \end{align}
  where the first equality follows from the base case of $\alpha$ with $(\cdot_R, \ud{\alpha})$, and the second equality follows by the definition of $\proc{EMT}$. 
  

  (2) Next, suppose that $\Delta \ge 1$ and the claim holds for all $\Delta' \le \Delta - 1$. Let 
  \begin{math}
    A := a(\M, S, F)
    = \sum_{t \in \M[0:\Delta]} \alpha(t, D)
  \end{math}
  be the solution. 
  Since $\M$ can be split as $\M = \M[1:\Delta] \uplus \M[0]$ such that $\M[1:\Delta]$ is the subset consisting of all composite trees, if we define
  \begin{align}\label{eq:1}
    A_0 &:= \sum_{t \in \M[0]} \alpha(t, D) = \proc{EMT}(0, S, F),
   \qquad 
   A_{1:\Delta} :=  \sum_{t \in \M[1:\Delta]} \alpha(t, D), 
  \end{align}
  then the solution $A$ equals the summation
  of $\alpha(t, S)$ over all trees $t$ in $\M$, we can decompose  $A$ into the sum $A = A_0\uplus A_{1:\Delta}$ of values $A_0$ and $A_{1:\Delta}$,
  where
  $\M[0] = L$. 
  We see that any pair $(t, S)$ of a composite tree $t = \nd{f, t_1, t_0} \in \M[1:\Delta]$ and a sample $S$ can be decomposed into smaller problems $(t_1, S_{f=1})$ and $(t_0, S_{f=0})$, where $S  = S_{f=1} \uplus S_{f=0}$.
  Applying the distributivity of $\cdot_R$ over $+_R$, we obtain the following derivation: 
  \begin{align}
    A_{1:\Delta}
    &= \sum_{t \in \M[1:\Delta]} \alpha(t, D, F)
  =
\makebox[10mm][l]{$
  \sum_{f \in F}
  \sum_{t_1}
  \sum_{t_0}
  \left(
  \begin{array}{l}
    \alpha(t_1, S_{f=1}, F\setmn f)
    \\ \cdot_R\; \alpha(t_0, S_{f=0}, F\setmn f)
  \end{array}
  \medstrut\right)
  \label{eqn:dist:one}
$}
  \\&= \sum_{f \in F} \sum_{t_1}  
  \left(\alpha(t_1, S_{f=1}, F\setmn f) \cdot_R C_0
  \rule{0pc}{1.0pc}\right)
  &\makebox[3cm][r]{$\because$ left-distributivity of $\cdot_R$}
  \\&= \sum_{f \in F} 
  \left(C_1 \cdot_R C_0 \right)
  &\makebox[3cm][r]{$\because$ right-distributivity of $\cdot_R$}
  \\
  &= \sum_{f \in F} 
  \left(
  \begin{array}{l}
  \proc{EMT}(\Delta-1, S_{f=1}, F\setmn f)
  \\ \cdot_R\; 
  \proc{EMT}(\Delta-1, S_{f=0}, F\setmn f)
  \end{array}
  \medstrut\right)
  &\makebox[3cm][r]{$\because$ induction hypothesis}
  \label{eq:emt:corr:ind:two}  
  \end{align}
  where $t_1, t_0$ range over composite trees in $\M[d-1]$, 
  $C_0 := \sum_{t_0} \alpha(t_0, S_{f=0})$, 
  and 
  $C_1 := \left(\sum_{t_1} \alpha(t_1, S_{f=1}) \right)$.
  As seen above, the left- and right-distributivities are used. 
  The last line follows from the induction hypothesis.
  Hence, the lemma is proved. 
\qed 
\end{proof}

Now, we show the first theorem. 
Recall that $O^*$ notation hides polynomial factors. 

\begin{theorem}[Complexity of $\pb{ADTC}$ without constraint] \label{thm:algo:emt:unconst}
  Let $(R, +_R, \cdot_R)$ be any semiring. 
  Given an unconstrained rank-$d$ query $Q = (\Delta, \alpha)$ and an input data $D = (S, y)$,
  the $\pb{ADTC}$ problem can be solved in 
  $O^*(n^{O(\Delta)}\cdot t_R)$ time and
  $O^*(s_R)$ space. 
\end{theorem}

\begin{proof}
  The correctness follows from Lemma~\ref{lem:emt:correctness}.
  For the time complexity, we observe that every iteration $X$ of the procedure is specified by the \textit{unique decision path} $P$ in $(F\by\set{1,0})^*$ shown in Algorithm~\ref{algo:emt}.
  From this, we see that there are at most $(2|F|)^\Delta |L| \le 2^\Delta n^\Delta l$ distinct decision paths of length $\Delta$. Since we can charge $O(t_R)$ time for operating on elements of $R$ to each iteration $X$, we see that the running time is bounded by $O((2^\Delta n^\Delta l)t_R)$.
  Furthermore, at each intermediate iteration $X$ of depth $\Delta' \le \Delta$, the procedure stores at most $2 \Delta'$ tensors on a stack, each of which occupies $O(N^{ck} s_R)$ space. Hence, the total space is $O(\Delta N^{ck} s_R)$. This shows the theorem. 
\qed   
\end{proof}

From \cref{thm:algo:emt:unconst}, we see that the problem $\pb{ADTC}$ is computable in polynomial time w.r.t.~the size of input $D$ when a query $Q$ is regarded as constant.



\subsection{\pb{ADTC} over a semiring of model behavior tensors}
\label{sec:method:mtensor}

Next, we present an efficient algorithm for the \textsc{Model Behavior Tensor} problem.
Intuitively, a rank-$d$ \textit{model behavior tensor} (or \textit{model tensor}, MT) is a $d$-way cross table of aggregation values via $\alpha$ on $\M$ w.r.t.~multiple metrics $\vsigma$ (see Fig.~\ref{fig:dashboard:size_dep}(a)).
An \textit{MT query} means any rank-$d$ query $Q = (\Delta, \alpha, \vsigma)$ with the empty constraint $\phi[\vsigma]$. 


\paragraph{Behavior of  decision trees as tensors.}
Let $\Delta, d \in \nat$ be any integers, and let 
$(R, +_R, \cdot_R, \Zero, \One)$ be any semiring.
Let $Q = (\Delta, \alpha, \vsigma, \phi[\vsigma])$ be any rank-$d$ (analytic) query.
Recall that $\S = \TS$ is the domain of structures. 
By the assumptions in \cref{subsec:mathod:notation}, we assume without loss of generality that the following conditions hold: 
\begin{itemize}
\item
(1)  The aggregation function $\alpha: \S \to R$ has a decomposition scheme $(\alpha, R)$ such that $\alpha = \HOM[\ud{\alpha}, \cdot_R]$ via operator $\cdot_R$ using a submonoid $(R, \cdot_R, \One)$ of $R$.
\item
(2) The combination function $\vsigma = (\sigma_i)_d: \S \to \nat[\vm]$, a $d$-vector of metrics with a shape vector $\vm  = (m_i)_d \in \nat^d$, has a decomposition scheme $(\vsigma, \A)$ such that $\vsigma = \HOM[\ud{\vsigma}, \vec\circ]$ via operator $\vec\circ$ using a $d$-vector of monoids $(\A, \vec\circ, \One)$, which denotes $\set{ (A_i, \circ_i, 1_i) }_d$.  
\end{itemize}

Under the above assumption,
we define the \textit{behavior} of each structure $\tau = (t, S) \in \S$ to be the \textit{index-weight pair} $(\vec k, v)$ computed by 
$\vec k = \vec{\sigma}(t, S) \in \nat[\vm]$ and
$v = \alpha(t, S) \in R$. 
Recall that a tensor $T$ in $R^{\nat[\vm]}$ is just a function $T: \nat[\vm] \to R$ that assigns an element $T(\vec k)$ in $R$ to each index (a point) $\vec k$ in the $d$-dim discrete space $\nat[\vm]\subseteq \nat^d$. 
Now, we state the second problem of this paper as follows. 

\begin{definition}[{\pb{MT}} problem]\rm\label{def:modeltensor}
Let $(R, +_R, \cdot_R, \Zero, \One)$ be any semiring and $(\A, \vec\circ, \One)$ be a $d$-vector of monoids. 
The \textsc{Model Behavior Tensor} problem over $R$ (\pb{MT}) is the problem of,
given any rank-$d$ model tensor query $Q' = (\Delta, \alpha, \vsigma)$ with decomposition schema $(\alpha, R)$ and $(\vsigma, \A)$, and input data $D = (S, y)$, 
computing the rank-$d$ tensor $T \in R^{\nat[\vm]}$
such that for every $\vec k \in \nat[\vm]$, the entry $T(\vec k)$ is defined as 
\begin{align}
  T(\vec k)
  &=
    \sum_{\substack{
    t \in \sig M_\Delta
  \\(t, S)\models \phi, \; 
  \vec k = \vsigma(t, S)
  }}
  \alpha(t, S)
  & \in R, 
\end{align}
where
$\vsigma(t, S)$ and $\alpha(t, S)$ are the
index and value of a structure $(t, S)$ in $\S$, respectively,
and $\sum$ is the summation $+_R$ over $R$.
Then, the tensor $T$ is called the \textit{model behavior tensor} (or a \textit{model tensor}) for $(Q', D)$ and denoted by $\pb{MT}[Q, D]$.
\end{definition}

A key to efficient algorithms for the problem is a set of operations over model behavior tensors, introduced as follows.

\begin{definition}[Tensor operations]\rm 
  We define the following operations. 
  Let $\vm \in \nat^d$ be any $d$-vector of integers, and $\S$ be any aggregation scheme with dimension $\vm$. 
  \begin{enumerate}
  \item The \textit{zero tensor} $\Zero$ that has zero $0_R$ everywhere, i.e., $\Zero(\vec k) = 0_R$  for all $\vec k$.
  \item The \textit{unity tensor} $\One$ that has a unity $1_R$ at
    the zero vector 
    $(0_i)_d$ and zero $0_R$ elsewhere. 
  \item The \textit{singleton tensor} $T = [\vec i = \vec k] v$, in \textit{Iverson's notation} with variable $\vec i$, for a tree-data pair $(t, S)$ that holds a semiring value $v = \alpha(t, S)$ at point $\vec k = \vsigma(t, S)$. 
  \item The point-wise addition $T_1\oplus T_0$ such that at every point $\vec k \in \nat[\vm]$, $(T_1\oplus T_0)(\vec k) = T_1(\vec k) +_R T_0(\vec k) \in R$.
    
  \item The (truncated) convolution product $T_1\otimes T_0$ such that at every point $\vec k \in \nat[\vm]$,
    \begin{align}
      (T_1\otimes T_0)(\vec k)
      &=
      \sum_{\substack{
      \vec k_1, \vec k_0 \in \nat[\vm]: \: 
      \vec k_1 \vec\circ\: \vec k_0 = \vec k
      }}
      T_1(\vec k) \cdot_R T_0(\vec k)
      &\in R, 
    \end{align}
    where $\vec k_1, \vec k_0$ range over $\nat[\vm]$, $\sum$ takes $+_R$ on $ R$, and $\vec\circ$ operates on $\nat[\vm]$. 
  \end{enumerate}
\end{definition}

In this way, we obtain from any semiring $(R, +_R, \cdot_R, 0_R, 1_R)$
the algebraic structure $(R^{\nat[\vm]}, \oplus, \otimes, \Zero, \One)$
over the collection $R^{\nat[\vm]}$ of rank-$d$ tensors from $\nat[\vm]$ to $R$.

\paragraph{Decision trees as polynomials.}
Any model tensor
$T \in R^{\nat[\vm]}$ maps to a multivariate truncated polynomial $P(x) = \sum_{\vec{k} \in \nat[\vm]} T(\vec{k}) x^{\vec{k}}$ over $R$, where $x = (x_i)_d$ are indeterminates. This establishes a bijection between $R^{\nat[\vm]}$ and the quotient semiring
$R[(x_i)_d] / \pair{x^{m_j}}\idx j1d$
bounding each degree to $m_i - 1$.%
\footnote{
  Unlike cyclic polynomials $R[(x_j)_d]/\pair{x^{m_j}-1}\idx j1d$ typical in signal processing, we use truncated polynomials $R[(x_j)_d]/\pair{x^{m_j}}\idx j1d$.
}
Division in $R$ is not required.

\begin{lemma}[Folklore]\label{lem:isomorphic:polynomials:tensors:truncated}
  Under the above correspondence, the set of $d$-variate truncated polynomials modulo
  $\pair{x^{m_j}}\idx j1d$
  and the set of rank-$d$ tensors with size $\vm$ are isomorphic.
  Specifically, the addition and product of polynomials coincide with
  the element-wise addition
  $T_1 \oplus T_2$ and the truncated convolution $T_1 \otimes T_2$ of tensors, respectively.
\end{lemma}

From Lemma~\ref{lem:isomorphic:polynomials:tensors:truncated}, the next lemma follows. 
\begin{lemma}[Lifting lemma]
  If $(R, +_R, \cdot_R, 0_R, 1_R)$ is a (commutative) semiring, then the structure 
  \begin{math}
(R^{\nat[\vm]}, \oplus, \otimes, \Zero, \One)
\end{math}
is a (commutative) semiring.
\end{lemma}

\paragraph{Reduction from {\pb{MT}} to Unconstrained {\pb{ADTC}}.}
We present a reduction from the \textsc{Model Behavior Tensor} problem to the unconstrained \pb{ADTC} problem over the tensor semiring $R^{\nat[\vm]}$.

\begin{definition}[Reduction from {\pb{MT}} to unconstrained {\pb{ADTC}}]\rm 
  Given an instance $Q' = (\Delta, \alpha, \vsigma)$ of the \pb{MT} and input data $D = (S, y)$, we let $\alpha_*: \S \to R^{\nat[\vm]}$ be the aggregation function defined by the singleton tensor $\alpha_*(t, S) = [\vec i = \vec k]v \in R^{\nat[\vm]}$ for all $(t, S) \in \S$. 
  The instance of the unconstrained \pb{ADTC} problem consists of the query $Q = (\Delta, \alpha_*)$ and the data $D$. 
\end{definition}

Unconstrained \pb{ADTC} computes the $\oplus$-sum of $\alpha_*(t, S)$ over $t \in \M$. Definitions~\ref{def:problem:adtc} and~\ref{def:modeltensor} immediately yield \cref{lem:coincide:unconstraint:mbtensor:problems}. 

\begin{lemma}[Correctness of the reduction]
  \label{lem:coincide:unconstraint:mbtensor:problems}
  Let $R$ be any commutative semiring. For any rank-$d$ model tensor query $Q = (\Delta, \alpha, \vsigma)$, the solution $\pb{MT}[Q, D]$ over the tensor semiring $(R^{\nat[\vm]}, \oplus, \otimes, \Zero, \One)$ coincides with the solution of the \pb{ADTC} problem using the constructed unconstrained query $Q' = (\Delta, \alpha_*)$.
\end{lemma}

\begin{theorem}[Complexity of {\pb{MT}}]
  \label{thm:algo:emt:tensor}
  The \textsc{Model Behavior Tensor} problem over $R$ can be solved in 
  $O^*(m_*^d \cdot n^{O(\Delta)} \cdot t_R)$ time and
  $O^*(m_*^d \cdot s_R)$ space,
  where $Q = (\Delta, \alpha, \vsigma)$ is a rank-$d$ tensor query, $D = (S, y)$ is input data, and ${m_*} = \max_i m_i$ is the maximum range value. 
\end{theorem}

\begin{proofsketch}
  Since the procedure is identical to Algorithm \proc{EMT}, the correctness and complexity analysis follow directly from \cref{thm:algo:emt:main}. We estimate the time and space complexities of tensor operations using the costs $t_R$ and $s_R$ of the underlying semiring $R$. By assumption, the maximum range size of the $d$ metrics is bounded by ${m_*}$. Thus, the volume (i.e., the number of effective entries) of any model behavior tensor $T$ is at most $O(m_*^d)$. Since each entry requires bounded space $s_R$, the result follows.
\qed   
\end{proofsketch}



\subsection{Putting it together}
\label{sec:method:together}
Given the model behavior tensor $T$, the solution to \pb{ADTC} is obtained by filtering out entries that violate $\phi[\vec k]$ and summing the weights $\alpha(t, S)$ of the remaining entries. Thus, the following theorem follows from \cref{thm:algo:emt:tensor}.

\begin{theorem}[Complexities of $\pb{ADTC}$ with general constraint]
  \label{thm:algo:emt:main}
  Given a general rank-$d$ query $Q = (\Delta, \alpha, \vsigma, \phi[\vsigma])$
  and an input data $D = (S, y)$,
  the general $\pb{ADTC}$ problem can be solved in
  the asymptotically same time and space complexities as
  \cref{thm:algo:emt:tensor}. 
\end{theorem}

Finally, the cell size $s_R$ and operation time $t_R$ over $R$ depend on the maximum entry size $M$ in $T$. In the worst case, $M$ reaches $O(2^\Delta)$ if all $|\M| = n^{\Theta(2^\Delta)}$ trees
fall 
into a single cell, yielding $s_R = \Theta(\log |\M|) = \Theta(2^\Delta \log n)$.
In practice, however,
appropriate hyperparameters like  
depth $\Delta$ and
\texttt{minsup}
effectively bound $M$ 
(\cref{sec:exp}).




\section{Extensions}
\label{sec:ext}

In this section, we introduce extensions to algorithm \pb{EMT} in \cref{sec:method} that enable the use of complex, nonlinear metrics as well as selection and sampling of concrete models. 

\subsection{Extensions to contingency tables and nonlinear model metrics}
\label{sec:ext:nonlinear:metrics}


Global analysis often involves complex \textit{nonlinear metrics} \cite{demirovic:aaai2021nonlinear:metrics}. To address this, we use \textit{rational metrics}, functions defined by a fraction $P(x)/Q(x)$ of multivariate polynomials over primitive metrics (\cref{subsec:prelim:constraint}),
Many statistical scores~\cite{shwartz:ben-david2014coltbook} can be expressed as rational metrics. 
Below, we introduce rational metrics definable in terms of \textit{2-way} and \textit{3-way contingency tables}~\cite{hastie2001eslbook}. Let indices $i, j, k \in \set{0, 1}$ represent the \textit{true label} $y$, \textit{predicted label} $\hat y = f_t(x)$, and \textit{sensitive attribute} $z$ in a dataset $S$, respectively. For a tree $t$, the 2-way table entry $N_{ij}(t; S)$ is defined as $\sum_{x\in S} \One[(y(x)= i) \land (f_t(x) = j)]$. Entries $N_{ijk}(t; S)$ of a 3-way table are defined analogously. Cell probabilities are then easily derived (e.g., true positives $tp = N_{11}/|S|$).

Using the contingency tables, we can formulate a variety of score functions \cite{shwartz:ben-david2014coltbook}, such as the F1-score~\cite{hastie2001eslbook} and Equalized odds (\texttt{EOdd})~\cite{hardt2016equality} as follows:
\begin{enumerate}[label=(\roman*)]
  \item \textit{F1-score}: $F1(t; S) = 2 N_{11}/(2 N_{11} + N_{01} + N_{10})$, and
  \item \textit{Equalized odds}: 
  $EOdd(t; S) = | \infrac{N_{011}}{N_{0*1}} - \infrac{N_{010}}{N_{0*0}} | + | \infrac{N_{111}}{N_{1*1}}$ $-$ $\infrac{N_{110}}{N_{1*0}} |$,
  where $N_{i*k} = N_{i1k} + N_{i0k}$ denotes the marginal count for true label $i$ and sensitive attribute $k$.
\end{enumerate}

As a practical optimization for \pb{ADTC}, we minimize the tensor dimensionality by isolating and maintaining \textit{model-independent metrics} outside a tensor $T$.
For instance, $N_{i*k}$ for any $i,k \in \set{0,1,*}$ is a model-independent metric because its prediction label index $j$ is fixed to $*$. 
The \textit{counts of all, positive}, and \textit{sensitive examples}, specifically $N_{***} = |S|$, $N_{1**} = N_{y=1}$, and $N_{**1} = N_{z=1}$, fall into this class.
This leads to significant savings in both memory and construction time.


\begin{figure}[t]
    \centering
    \begin{minipage}[t]{0.48\linewidth}
        \vspace{0pt}
        \centering
        \includegraphics[height=0.85\linewidth]{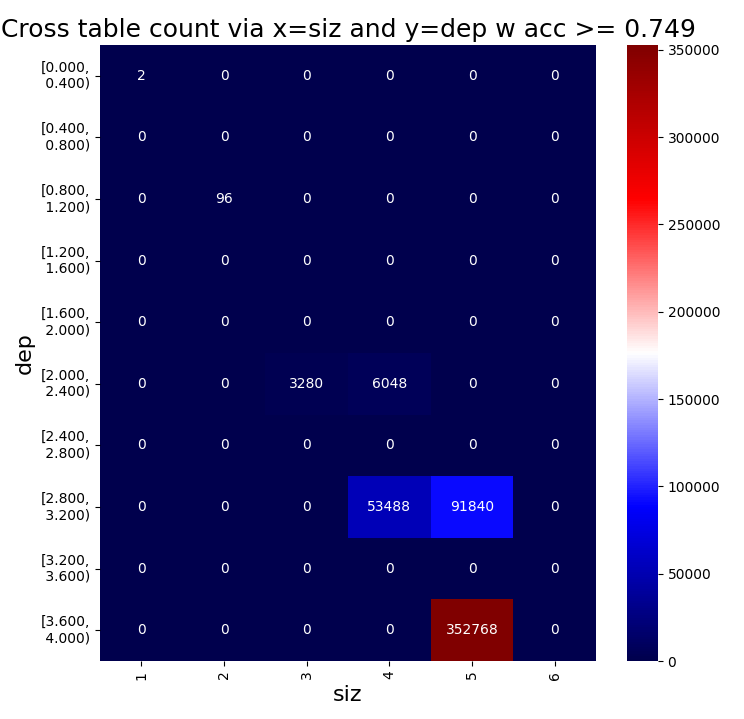}
    \end{minipage}
    \hfill
    \begin{minipage}[t]{0.48\linewidth}
        \vspace{0pt}
        \centering
        \includegraphics[height=0.85\linewidth]{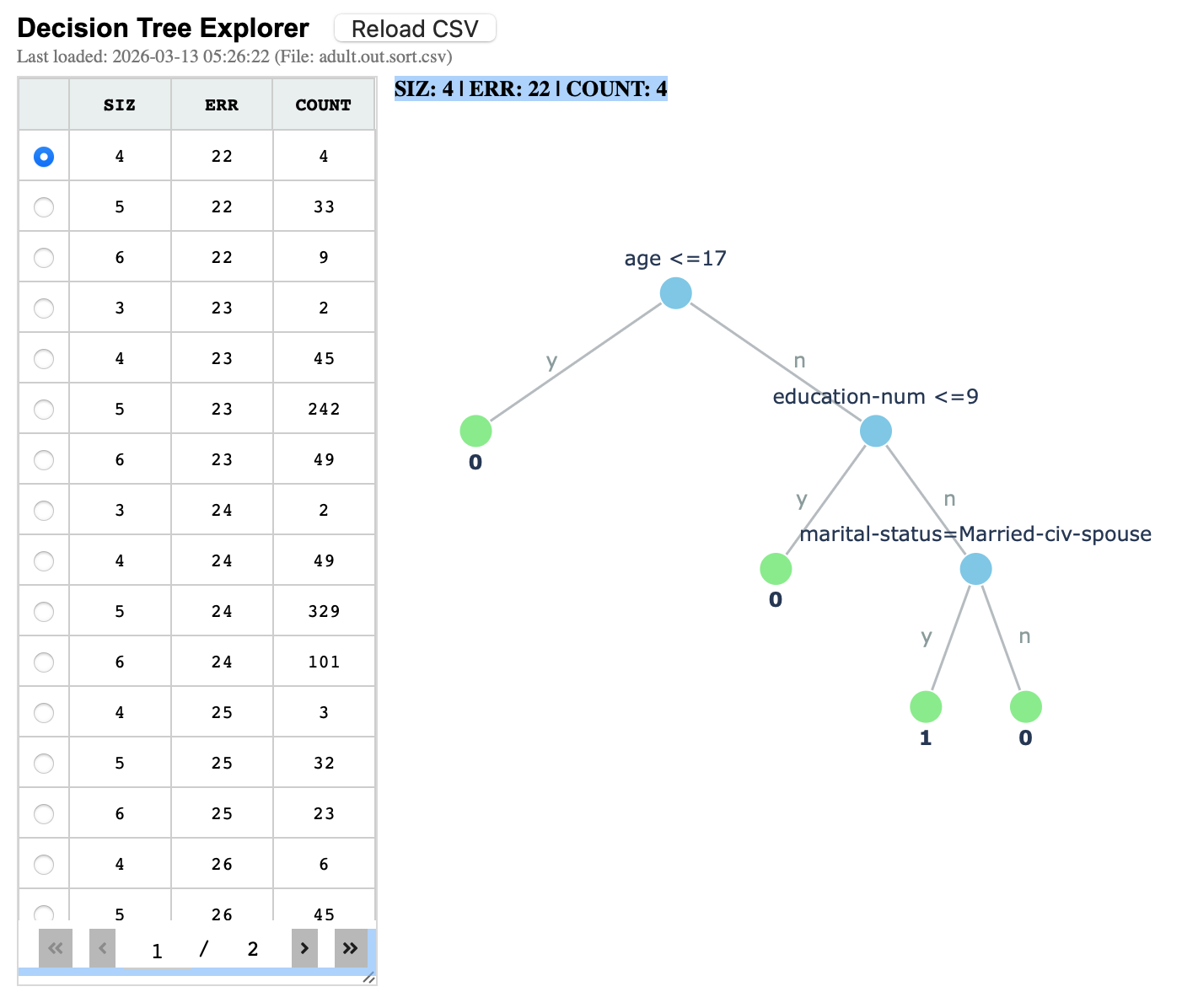}
    \end{minipage}
    \vspace{1mm}
    
    \begin{minipage}[t]{0.48\linewidth}
        \centering
        \small (a) Model size vs. depth ($\text{acc} \ge 0.749$)
    \end{minipage}
    \hfill
    \begin{minipage}[t]{0.48\linewidth}
        \centering
        \small (b) Dashboard interface
    \end{minipage}
    \vspace{1mm}
    \caption{(a) Model profiles generated by \pb{ADTC} on the \texttt{adult} dataset showing the structural distribution of near-optimal models
      (\cref{subsec:exp:impact}).
      (b) Dashboard interface for selecting a model from the cross-table to display its decision tree
      (\cref{subsec:method:sampling}). 
    }
    \label{fig:dashboard:size_dep}
    \vspace{-3mm}
  \end{figure}

\subsection{Selection and sampling of supporting models}
\label{subsec:method:sampling}

By modifying the algorithm \proc{EMT}, we can efficiently implement the selection and random sampling of decision trees that support the value of ADTC(R) in a model behavior tensor $T$.
For this purpose, we use the technique of a \textit{semiring extension} for selection and sampling, recently proposed by Goral \textit{et al.}~\cite{goral:giesen:blacher:staudt:kaus:aaai2024model} (see also \cite{arimura:osabe:uno2017ifors} for a similar technique).
To do this, we construct the product semiring $\mathbb{N}\times\mathbb{T}$ of the count semiring $\mathbb{N}$ and the pseudo semiring $\mathbb{T}$ of decision tree syntaxes.
For each index $\vec k \in \mathbb{N}[\vm]$,
the $\vec k$-th cell of $T$ holds a pair $T(\vec k) = (N_{\vec k}, t_{\vec k})$ consisting of a count $N_{\vec k} \in \mathbb{N}$ and a tree $t_{\vec k} \in \M$. 
The first semiring $\mathbb{N} = (\mathbb{N}, +, \times)$ maintains the count $N_{\vec k}$ of trees falling into the $\vec k$-th cell of $T$, $N_{\vec k} = T(\vec k) \in \mathbb{N}$, for all indices $\vec k \in \mathbb{N}[\vm]$. 
On the other hand, the second pseudo semiring $(\mathbb{T}, +, \set{\pair{f,\cdot,\cdot}}_{f \in F})$ holds a tree $t$ uniformly sampled by
the probability $P_{\vec k} = N_{\vec k}/N_\idrm{total} \in [0,1]$,
where $N_\idrm{total} = \sum_{\vec k} N_{\vec k} = |\M|$ is the total count. Since the rest of the construction is almost the same as Goral \textit{et al.}~\cite{goral:giesen:blacher:staudt:kaus:aaai2024model}, we omit the details.
We implemented this function in our software \texttt{emtrees}. Fig.~\ref{fig:dashboard:size_dep} shows our dashboard interface, rendering a randomly sampled
tree from the associated cell.


\begin{figure}[t]
\centering
\begin{minipage}[t]{0.85\linewidth}
  \centering
  \includegraphics[width=\linewidth]{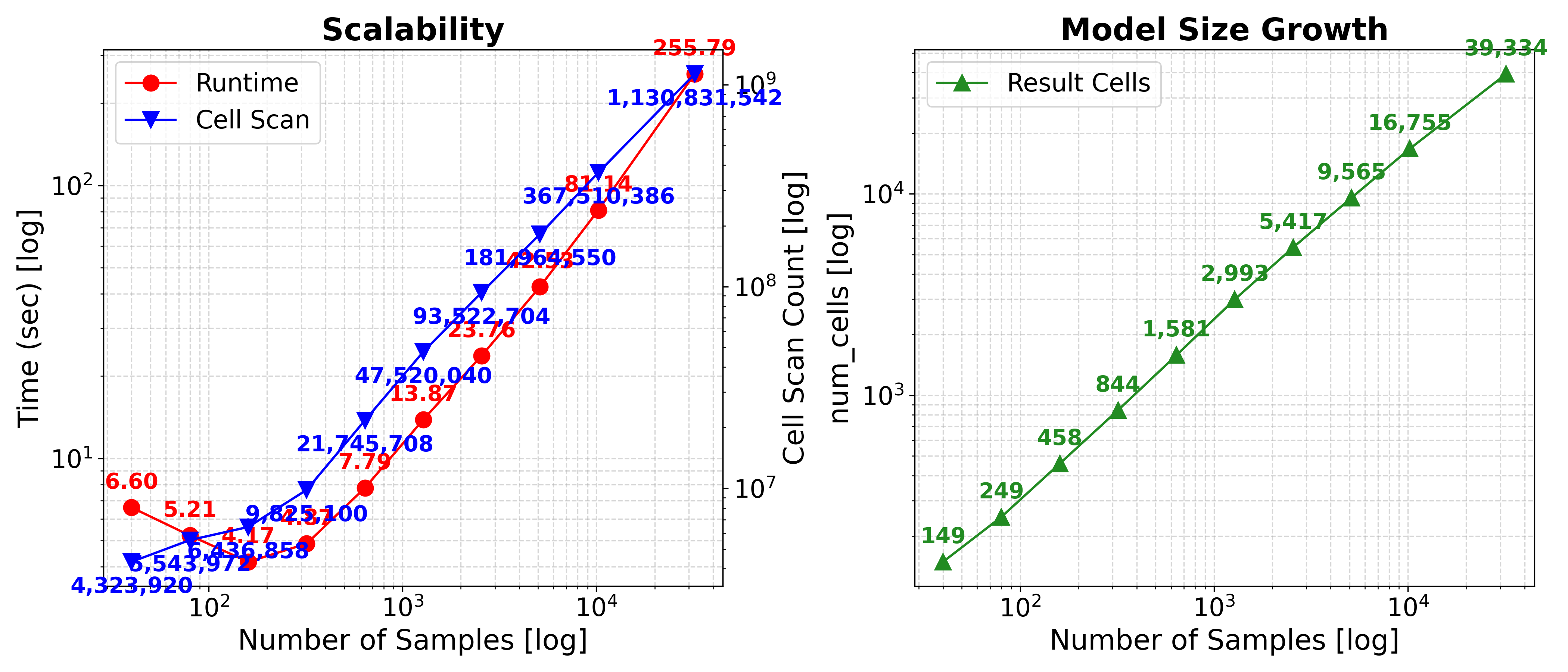}
  \caption{Scalability of \pb{ADTC} relative to data size (\cref{subsec:exp:scalability}).}
  \label{fig:scalability}
\end{minipage}
\end{figure}

\section{Experimental Evaluation}
\label{sec:exp}
We evaluate \pb{ADTC} on real-world data to demonstrate its scalability and utility in providing a transparent and rigorous analysis of
high-performing decision trees.

\subsection{Setup}
\textbf{Dataset.} We use the UCI \texttt{adult} dataset
(32,129 data) as a standard benchmark for research on fairness and predictive multiplicity, aiming to predict whether annual income exceeds \$50,000 using mixed categorical and continuous variables.

\textbf{Measurement Protocol.}
\textit{Computational cost} is evaluated using: (i) \texttt{time} (execution time in seconds); (ii) \texttt{\texttt{cell\_scan}} (total number of cells scanned as a platform-independent cost); and (iii) \texttt{num\_cells} (non-zero entries in the resulting tensors).
\textit{For parameter configurations},
we use the following settings unless otherwise stated: 
$\texttt{maxdep}=5$ (maximum depth), 
$\texttt{nbins}=7$ (number of bins in discretization of numerical features), 
$\texttt{relminsup}=0.15$ (relative minimum support), 
and $\texttt{siz}=6$ (tree size).
\textit{Data scalability analysis} varies $n_{\text{data}}$ from 40 to 32,129 (all data). 
\textit{Structural complexity analysis} scales \texttt{maxdep} and \texttt{siz} proportionally. 
\textit{Constraint relaxation} analysis varies \texttt{siz} from 1 to 16 with $n_{\text{data}}=640$ and $\texttt{maxdep}=4$.
\paragraph{Environment}:
The system \texttt{emtrees} is implemented in Python 3.12 and executed on an Apple M1 Pro PC (16GB memory, macOS 15.7.1).
The source code is publicly available.
\footnote{\texttt{emtrees}:~\url{https://doi.org/10.5281/zenodo.20842908}}

\begin{figure}[t]
\centering  
\begin{minipage}[t]{0.85\linewidth}
  \centering
  \includegraphics[width=\linewidth]{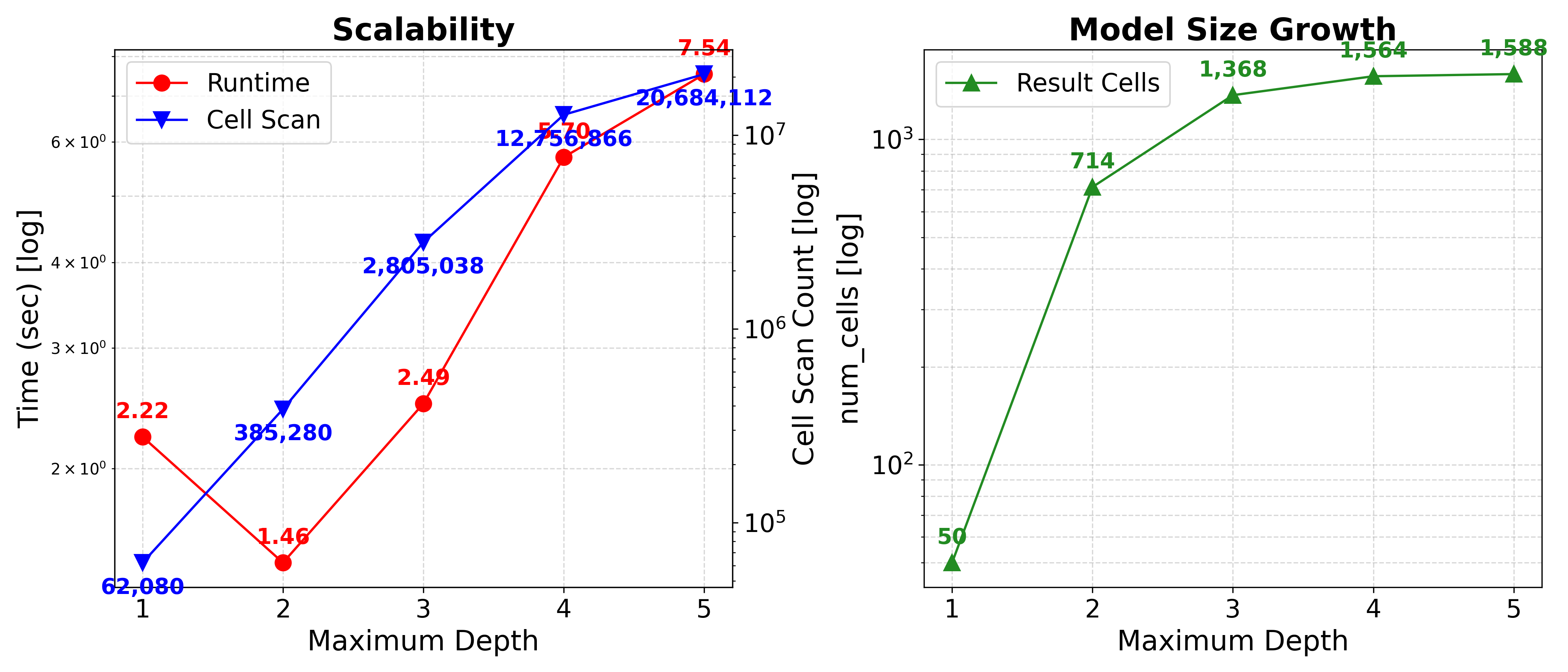}
  \caption{Impact of maximum depth on efficiency (\cref{subsec:exp:scalability}).}
  \label{fig:depth_impact}
  \centering
  \includegraphics[width=\linewidth]{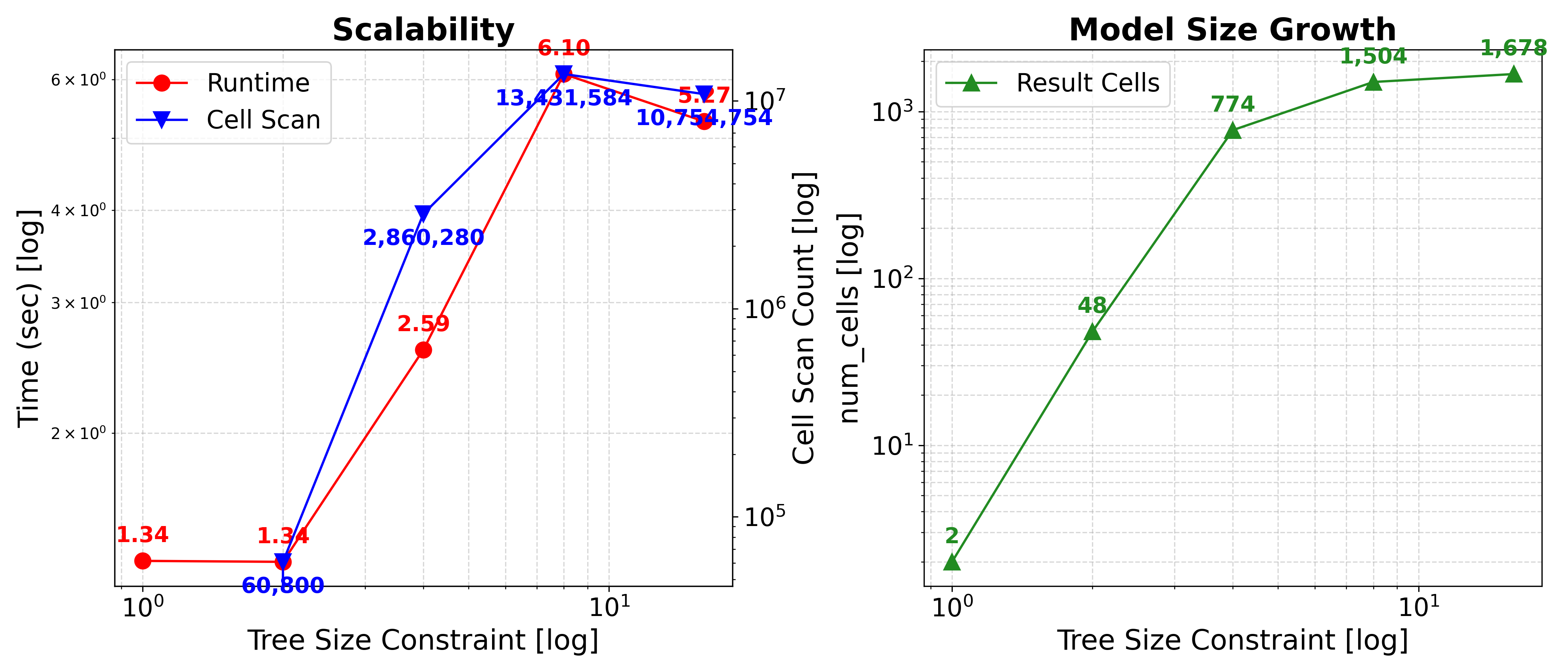}
  \caption{Effect of tree size on resolution (\cref{subsec:exp:scalability}).}
  \label{fig:size_impact}
\end{minipage}
\vspace{-3mm}
\end{figure}

\subsection{Scalability and computational efficiency}
\label{subsec:exp:scalability}
Our empirical evaluation demonstrates that \pb{ADTC} exhibits scalable and stable behavior across all parameters.
Our program ran within several minutes on mid-sized datasets such as UCI \texttt{mushroom}, \texttt{adult},%
\footnote{\url{https://archive.ics.uci.edu}}
and \texttt{COMPAS} 
\footnote{\url{https://www.kaggle.com/datasets/danofer/compass}}
under the
parameter settings above. 

\textbf{Scalability with input size.}
As shown in Fig.~\ref{fig:scalability}, the runtime and \texttt{cell\_scan} follow near-linear trends on a log-log scale, confirming polynomial complexity relative to $n_{\text{data}}$.
Even with an 800-fold increase in data size, the runtime remains manageable
at 255.79 seconds for the full dataset
($1.1\times 10^9$ cells scanned).
Concurrently, \texttt{num\_cells} grows proportionally, indicating a more refined model profile for larger datasets.

\textbf{Impact of structural constraints.}
We evaluate how structural expansion affects efficiency. 
\textit{Tree depth}: Fig.~\ref{fig:depth_impact} shows that the \texttt{cell\_scan} grows as depth $\Delta$ (\texttt{maxdep})
relaxes, reflecting the expansion of the search space. Despite the doubly exponential nature of the syntactical space, the operations remain stable, reaching $\approx 2.0 \times 10^7$ scanned cells at $\Delta = 5$.
\textit{Tree size constraint}: Fig.~\ref{fig:size_impact} shows that relaxing the tree size constraint $s$ (\texttt{siz}) increases \texttt{num\_cells}, yielding a more detailed model profile.
Stabilization of \texttt{cell\_scan} at higher \texttt{siz} values suggests that model tensors efficiently aggregate properties once structural complexity is captured, incurring minimal overhead.

\subsection{Impact of model complexity on the Rashomon set}
\label{subsec:exp:impact}

We analyze the structural properties of near-optimal models by examining the relationship between model size (\texttt{siz}) and depth (\texttt{maxdep}). We focus on high-performing models with $\text{acc} \ge 0.749$ in the Rashomon set.
Fig.~\ref{fig:dashboard:size_dep}~(a) (in \cref{subsec:method:sampling}) reveals that 352,768 top-performing trees cluster in the range of depth $4$ and sizes $4\sim 5$, showing strong structural consistency.
This highlights the trade-off between interpretability and predictive performance under diverse deployment scenarios.
See \cref{subsec:method:sampling} for details on the dashboard shown in Fig.~\ref{fig:dashboard:size_dep}~(b).

\subsection{Trade-off analysis between accuracy and fairness}
\label{subsec:exp:tradeoff}
%
%
To demonstrate how \pb{ADTC} supports evidence-based model selection, Fig.~\ref{fig:profile:acc-fone:acc-dpgap} (in \cref{sec:intro}) and Fig.~\ref{fig:profile:acc-eodd:acc-eopp} show the distribution of the 34,706 models extracted from the Rashomon set on the \texttt{adult} dataset. These distributions are evaluated with respect to \textit{accuracy} (\texttt{acc}), \textit{F1-score} (\texttt{F1}), and fairness metrics, including \textit{demographic parity} (\texttt{DP}), \textit{equalized odds} (\texttt{EOdd}), and \textit{equal opportunity} (\texttt{EOpp})~\cite{hardt2016equality}.
The obtained cross-tables provides a rigorous model profile that serves as objective evidence for model selection.

For instance, in Fig.~\ref{fig:profile:acc-eodd:acc-eopp}~(a), 
our framework identifies a \textit{large cluster of $15,340$ models} in the bin with the \textit{highest accuracy ($\texttt{acc} \in [0.765, 0.831]$) and low fairness error}
$\texttt{EOdd} \in [-0.029, 0.086]$),
while revealing a \textit{smaller cluster of $688$ models} with
the \textit{same accuracy but significantly higher fairness error}
($\texttt{EOdd} \in [0.886, 1.000]$).
Conversely, many models with moderate accuracy exhibit high fairness disparity,
such as the 10,906 models
with $\texttt{acc} \in [0.566, 0.633]$ and high \texttt{EOdd}. 
By providing a complete landscape of the Rashomon set, \pb{ADTC} facilitates evidence-based model selection, allowing practitioners to prioritize high-accuracy models with minimal fairness disparity from numerous candidates rather than relying on a single heuristic output.

\begin{figure}[t]
  \centering
  \begin{subfigure}[t]{\profilewidth}
    \centering
    \includegraphics[width=0.95\linewidth]{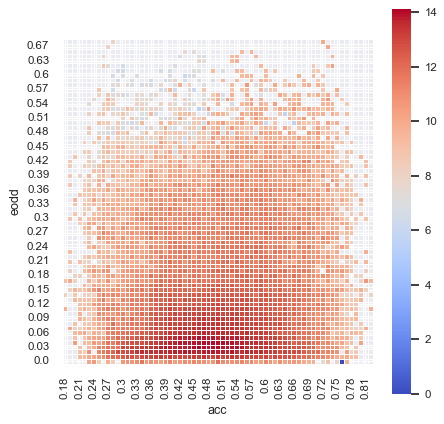}
    \caption{Accuracy vs. EOdd gap.}
  \end{subfigure}
  \hfill
  \begin{subfigure}[t]{\profilewidth}
    \centering
    \includegraphics[width=0.95\linewidth]{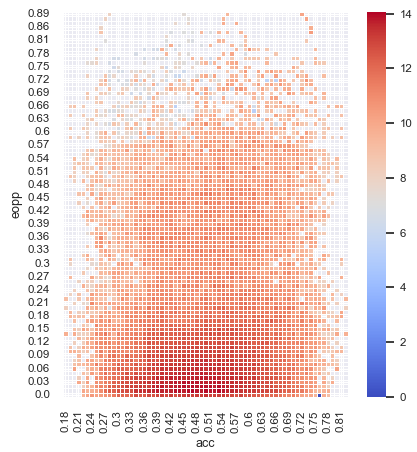}
    \caption{Accuracy vs. EOpp gap.}
  \end{subfigure}
  \vspace{-2mm}
  \caption{Model profiles generated by \pb{ADTC} on the \texttt{adult} dataset. Horizontal axes denote accuracy; vertical axes represent (c) EOdd and (d) EOpp gaps. Higher values are better for
    accuracy, 
    whereas lower values indicate better fairness (\cref{subsec:exp:tradeoff}).
  }
  \label{fig:profile:acc-eodd:acc-eopp} 
  \vspace{-3mm}
\end{figure}

\subsection{Assessment of preprocessing}
\label{subsec:exp:assessment}
In Fig.~\ref{fig:profile:acc:bacc} (in \cref{sec:intro}), 
we show the model profiles on \textit{accuracy} (\texttt{acc}) and \textit{balanced accuracy} (\texttt{bacc}) \cite{shwartz:ben-david2014coltbook} of near-optimal models generated by \pb{ADTC} on the original \texttt{adult} dataset and its class-balanced dataset.
They
show a discrepancy between \texttt{acc} and \texttt{bacc}
in the original dataset, which is successfully calibrated in the balanced dataset.

\section{Conclusion}
This paper presents \pb{ADTC}, a formal framework for the global assessment of decision tree spaces. Through the use of model behavior tensors and an algebraic formulation, we provide a scalable methodology for evidence-based model selection, enabling a rigorous analysis of the landscape of interpretable models. 

Theoretically, this research contributes by formulating the \pb{ADTC} problem and developing efficient algorithms that establish complexity upper bounds over general semirings.
For future work, extending the \pb{ADTC} framework to other interpretable models is a promising direction, building upon existing enumeration techniques for
\textit{LASSO models} \cite{hara2017enumerate},
\textit{support vector machines} \cite{kanamori2019enumeration}, and
\textit{rule lists with profile construction} \cite{mata2022computing}.
Finally, we will investigate the computational complexity of \pb{ADTC} relative to standard counting classes like Valiant's $\# \idrm{P}$ \cite{valiant1979sharpp:complexity} and the semiring-based class $\idrm{NP}_\infty(R)$ recently proposed in \cite{BDEKNP2025fagins}.
Our tensor-manipulation software, \texttt{emtrees}, introduced in Sec. 3 and Sec. 4, is publicly available at \url{https://doi.org/10.5281/zenodo.20842908}. 

\subsubsection*{Acknowledgments.}
The author thanks Koji Tsuda and Jun Sese for initially drawing attention to the optimal decision tree problem,
and Yasuko Matsui for invaluable discussions on our preliminary research~\cite{arimura:osabe:uno2017ifors} during the organized session at IFORS 2017 in Québec.
The author is also grateful to
Ichigaku Takigawa, 
Shinya Takamaeda-Yamazaki, 
Atsuyoshi Nakamura, 
Masato Motomura, 
and the members of the CREST project. 
The author is also grateful to 
Kazuki Yoshizoe,
Yasuaki Kobayashi, 
Norihito Yasuda, 
Takeaki Uno, and 
Shinichi Minato,
during the AFSA project meeting for their comments and discussions.
This work is supported by
MEXT/JSPS KAKENHI Grant Number
26K02980, 
20H00595, 
20H05963, 
and
JST CREST
18070962. 



%
%
\bibliographystyle{splncs04}
\bibliography{dtree,rasho,uni,npsr}
%

\end{document}